\documentclass[letterpaper, 10 pt, conference]{ieeeconf}  

\IEEEoverridecommandlockouts                              

\usepackage{times}
\usepackage{graphicx} 
\usepackage{algorithm}
\usepackage{algorithmic}
\usepackage{hyperref}
\usepackage{cite}
\usepackage{url}
\usepackage{scalefnt}
\usepackage[normalem]{ulem}
\usepackage{bm}
\usepackage{caption}
\usepackage{subcaption}
\newtheorem{theorem}{Theorem}

\usepackage[table]{xcolor}

\title{\LARGE \bf Cross-Domain Transfer in Reinforcement Learning using Target Apprentice}

\author{Girish V Joshi and Girish Chowdhary
	\thanks{*This work was supported by AFOSR FA9550-15-1-0146}
	\thanks{To appear as conference paper in ICRA 2018} \thanks{Authors are with the Department of Agriculture and Biological Engineering and Coordinated Science Lab, University of Illinois,
		Urbana-Champaign, IL, USA
		{\tt\small girishj2@illinois.edu,girishc@illinois.edu}}%
}

\begin{document}
	
\maketitle


\begin{abstract}
In this paper, we present a new approach to Transfer Learning (TL) in Reinforcement Learning (RL) for cross-domain tasks. Many of the available techniques approach the transfer architecture as a method of speeding up the learning target task. We propose to adapt and reuse the mapped source task optimal-policy directly in related domains. We show the optimal policy from a related source task can be near optimal in target domain provided an adaptive policy accounts for the model error between target and source. The main benefit of this policy augmentation is generalizing policies across multiple related domains without having to re-learn in the new tasks. Our results show that this architecture leads to better sample efficiency in the transfer, reducing sample complexity of target task learning to target apprentice learning. 

\end{abstract}


\section{Introduction}
\label{introduction}
Reinforcement Learning is a machine learning paradigm, where a robotic agent learns the optimal policy for performing a sequential decision making without complete knowledge of the environment. Recent successes in deep reinforcement learning have enabled RL agents to solve complex problems from balancing inverted pendulums \cite{sutton1998reinforcement} to playing Atari games \cite{mnih2015human}. Despite these recent successes, we do not yet understand how to efficiently transfer the learned policies from one task to another \cite{taylor2009transfer}. In particular, while some level of success has been achieved, in transferring RL policies in the same state-space domains, the problem of efficient cross-domain skill transfer is still quite open.

 We consider the term ``similar'' source and target tasks, in the sense that they exploit the same underlying physical principles, but their state spaces can be entirely different. For example in our primary results, we consider the problem of knowledge transfer from balancing Cart-Pole to Bicycle balancing. While both the system share common dimensionality of state and action spaces but span across a different coordinate frame. The Cart-pole is defined over states of cart and pendulum, $(x,\dot x, \theta, \dot \theta)$ and the action space is lateral force on cart $(-F,0,F)$. Whereas bicycle dynamics is modelled over handlebar rotation and bicycle roll angle $(\theta, \dot \theta,\omega,\dot \omega)$ with action space is torque applied by rider on the handlebar $(-\tau, 0, \tau)$. While the dynamics and domain of state space of two processes might be completely different, they share a commonality in the underlying physical principles. Both the systems exhibit non-minimum phase dynamics and also the nature of the control policy is same, such that the control action is applied, in the direction of the fall of pendulum or bicycle. This similarity in dynamical behavior of two systems makes learning in a cart-pole domain relevant in bike balancing problem.
 
Another cross-domain transfer results we present, is the problem of transferring the skills learned from the Mountain Car (Figure \ref{fig:MC})  \cite{sutton1998reinforcement} to the Inverted pendulum (Figure \ref{fig:IP}). In the Mountain Car, the agent learns the optimal policy to make an underpowered car climb a mountain. On the other hand, in the pendulum domain, the agent learns to balance a pendulum upright from its initial down position. In both cases, the common physical principle the agent must learn, is to exploit the principle of energy exchange. The pendulum must be swung up to the upright position by creating enough angular momentum through smaller oscillations, and similarly, the car must be made to climb a steeper slope by using energy exchanged by moving up and down the slope of the mountain. In principle, a good RL agent would find it easier to balance a pendulum after it has learned a related task, to make an underpowered car climb a mountain. 
 
Humans are capable of efficiently and quickly generalizing the learned skills between such related tasks. However, RL algorithms capable of performing efficient transfer of policies without learning in the new domain have not yet been reported.
\begin{figure}[tbh]
    \centering
    \begin{subfigure}{0.3\columnwidth}
        \includegraphics[width=\textwidth]{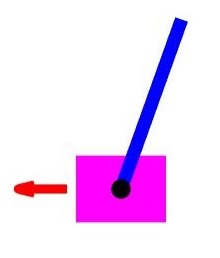}
        \vspace{-0.2in}
        \caption{}
        \label{fig:MC}
    \end{subfigure}
    \begin{subfigure}{0.1\columnwidth}
        \includegraphics[width=\textwidth]{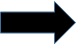}
        \vspace{-0.2in}
    \end{subfigure}
    \begin{subfigure}{0.5\columnwidth}
        \includegraphics[width= \textwidth]{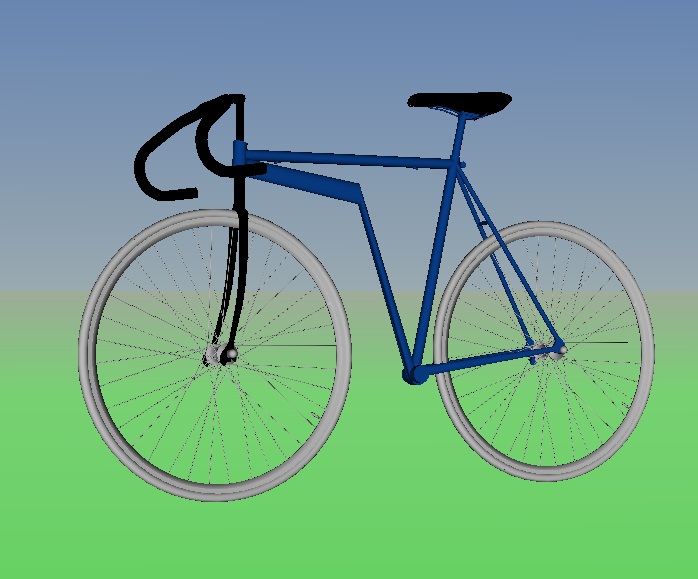}
        \vspace{-0.2in}
        \caption{}
        \label{fig:IP}
    \end{subfigure}
    \vspace{-0.12in}
    \caption{Cross Domain Transfer: (a) Source Task: Cart-Pole and (b) Target Task:Bicycle, VRML in MATLAB\textregistered  is used as simulation environment \cite{bicycle_sim} to demonstrate cross domain transfer.} 
    \label{fig:Environment}
\end{figure}
To address this gap, our main contribution is an algorithm that enables cross-domain transfer in the RL setting. Leveraging notions from apprenticeship learning \cite{abbeel2005exploration} and adaptive control \cite{aastrom2013adaptive,chowdhary2013rapid}, we propose an algorithm that can directly transfer the learned policy from a source to target task.

Given a source task and it's optimal policy, a target apprentice model and an inter-task mapping, we show that it suffices only to execute greedy policies augmented with an adaptive policy to ensure $\epsilon-$optimal behavior in target space.

\subsection{State of the Art: Transfer Learning in RL} 
A significant body of literature in transfer learning in the RL, are focused on using the learned source policy as an initial policy in the target task \cite{taylor2005value,ammar2012reinforcement,ammar2015unsupervised}. Examples include a transfer in scenarios where the source and target task are similar, and no mapping of state space is needed \cite{banerjee2007general}, or transfer from human demonstrations \cite{peters2006policy}. However, when the source and target task have different state-action spaces, the policy from source cannot be directly used in the target task. In this case, a mapping is required between the state-action space of the corresponding source and target tasks to enable knowledge transfer \cite{taylor2009transfer}. The inter-task mapping can be supervised; provided by an agent \cite{torrey2007relational}, hand coded  using semantics of state features \cite{liu2006value,banerjee2007general,konidaris2007building} ,unsupervised using Manifold Alignment \cite{wang2009manifold,ammar2015unsupervised} or Sparse Coding Algorithm \cite{ammar2012reinforcement}. Aforementioned TL methods accelerate the learning and minimize regret as compared to stand alone RL on target. However, with inter-task mapping and simple initializing of the target task learning with the transferred source policy may not lead to sample efficiency in the transfer. In particular, these approaches do not leverage the fact that both tasks exploit the same physical principle and the possibility of reusing source policy in the target domain.
\subsection{Main Contributions}
In this paper, we take a different approach from using the source policy to initialize RL in the target.  Inspired by the literature in model reference adaptive control (MRAC) \cite{chowdhary2013rapid} we propose an algorithm that adapts the source policy to the target task. Also, unlike MRAC literature we can extend the method to probabilistic MDPs with discrete state-action spaces. We argue that optimal policies retain its optimality across domains that leverage the same physical principle but different state-spaces. We augment the transferred policy by a policy adjustment term that adapts to the difference between the dynamics of two tasks in target space. If adaptive policy could be design to match the target model to projected source model, we demonstrate the adapted projected policy to be $\epsilon$-optimal in target task. The adaptive policy termed as $\pi^{(T)}_{ad}(s)$ is designed to accommodate the difference between the transition dynamics of the projected source and the target task. The key benefit of this method is that it obviates the need to learn a new policy in the target space, leading to high sample efficient transfer.  
\section{Transfer Learning with Target Apprentice (TA-TL)}
This section proposes a novel transfer learning algorithm capable of cross-domain transfer between two related but dissimilar tasks. The presented architecture applies to both continuous, discrete state and action space systems. Unlike the available state of art TL algorithms, which mainly concentrates on policy initialization in target task RL; we propose to use source policy directly as the optimal policy in the related target model. We achieve this one step transfer through online correction of transferred policy with adaptive policy, derived based on model transition error. The presented approach has three distinct phases: Phase I involves finding an optimal policy for source task. For this purpose, we use Fitted Q-Iteration (FQI) RL to solve the source MDP. Since the source task is much simpler and smaller problem compared to the target tasks, we assume we can always discover an optimal policy for source task. Phase II; involves discovering a mutual mapping between state, action space of source and target using Unsupervised Manifold Alignment (UMA). Phase III of transfer, is the adaptation of the mapped source optimal policy through policy augmentation in a new target domain.

A possible drawback of the proposed method could be suboptimal transfer, if the adaptive policy fails to account for the total model error between the projected source and target model. For given any small ``$\delta(\epsilon)$'' residual model error by adaptive policy, results in the suboptimal behavior in target task. With the advantage of high sample efficiency of the proposed technique, near-optimal behavior in the target can be acceptable. It is to be noted; we do not engage in exploration in target space for transfer, we only exploit the projected source policy in target space, to achieve near-optimal behavior. Nevertheless, with further exploration, we can improve upon the adapted transferred policy and achieve an optimal solution, but this is left for follow-on work.

Details of three phases of proposed transfer learning technique using target apprentice model are as follows:
\subsection{Phase I: Learning in the Source task} 
Fitted Q-Iteration is used to learn optimal policy in the source task. The policy search is not limited to Q-learning and can be extended to any other optimal policy generation methods. A single layer shallow network is used to approximate Q function. For the tasks considered in this paper, shallow networks were found sufficient, but for more complex tasks a deep architecture with multiple layers can be used. This exercise is underway and left for follow-on work.
\subsection{Phase II: Inter task Mapping}
\label{MA}
Transfer in RL setting, the source and target task have a different representation of state and action spaces. The cross-domain transfer requires an inter-task mapping to facilitate a meaningful transfer. State space $s^{(S)}$ and $s^{(T)}$ belonging to two different manifold, cannot be directly compared. Unsupervised Manifold Alignment (UMA) technique helps to discover alignment between two data sets and provide a one to one and onto inter-task mapping. Using this inter-task mapping, allows us to build one step optimal policy for target task. It is important to note we consider same cardinality and analogous action spaces in analysis and experiments for ease of exposition of the proposed transfer architecture. Problems with distinct, nonuniform action spaces will have to use classification methods to find correspondence between action spaces \cite{taylor2007cross,taylor2007transfer}. The transfer is achieved through augmenting transferred policy by the adaptive policy learned over the target model. The proposed policy transfer and adaptation method reuse the source policy in target space resulting in near-optimal behavior. Details of the inter-state mapping are provided in \cite{ammar2015unsupervised,wang2009manifold} and reference therein.
\begin{algorithm}
    \caption{Transfer Learning using target Apprentice model}
    \label{alg:TL}
    \begin{algorithmic}[1]
        \STATE {\bfseries Input:} Source Policy $\pi^{*(S)}(s)$, Inter-task mapping $\mathcal{\chi_S}$ and Apprentice Model $\hat \mathcal{P}^{(T)}$
        \REPEAT
        \STATE Initialize $(s_0,a_0)^{T} \in (\mathcal{S} \times \mathcal{A})^{T}$.
        
        \STATE Project the target task state to source model using inter-task mapping ${\chi_s}^+$
        $$\hat{s}^{S}_i = {\chi_s}^+(s^{T}_i)$$
        \STATE Evaluate the action using greedy policy on learned action-value function $Q^{(S)}(s,a)$ in source task,
        $$\pi^{*(S)}_i = arg\max_{\pi}(Q^{(S)}(\hat{s}^{S}_i,a))$$
        \STATE $$a^{S}_i =  \pi^{*(S)}(\hat{s}^{S}_i)$$
        \STATE Query the source task model at state $\hat{s}^{S}_i$ and action $a^{S}_i$ $$s^{S}_{i+1} = \mathcal{P}^{(S)}(\hat{s}^{S}_i,a^{S}_i)$$
        \STATE Project the source task propagated state to target task model,
        $$\hat{s}^{T}_{i+1} = {\chi_s}({s}^{S}_{i+1})$$
        \STATE Evaluate the adaptive policy as
        $$\pi^{(T)}_{ad} = \hat \mathcal{P}^{(T)}(s^{T}_i,a^T_i) - \hat{s}^{T}_{i+1}$$
        \STATE Project the source policy into target space $\mathbf{\chi}_s(\pi^{*(S)}(s,a))$
        \STATE TL policy for target task
        $$\pi^{*T} = \pi^{*(S)}(\chi_s (s^S)) - \pi^{(T)}_{ad}$$
        \STATE Draw action from policy $\pi^{*T}$ at step $i$ and propagate the target model
        \UNTIL{$s^{(T)}_i = terminal$} 
    \end{algorithmic}
\end{algorithm}
\subsection{Phase III: Transfer Learning through policy adaptation}
This section presents a transfer algorithm for a pair of tasks in continuous/discrete state and action spaces. Algorithm \ref{alg:TL} details TL through policy adaptation using apprentice model. Empirically we show presented method is sample efficient compared to other methods of transfer; since the sample complexity of learning an optimal policy for initialized target task is reduced to sample complexity of local apprentice model learning.
Algorithm-\ref{alg:TL} leverages the inter-task mapping detailed in subsection \ref{MA}, to move back and forth between source and target space for knowledge transfer and adaptive policy learning. The performance of the policy transfer depends on the quality of manifold alignment between source and target tasks. We assume UMA provides a one to one and onto correspondence between source and target state spaces for efficient transfer. Algorithm-\ref{alg:TL}, provides pseudo-code for TL using target apprentice. Steps $1-8$ provide an architecture for cross-domain policy transfer and step $9-12$ details policy adaptation through target apprentice learning.

\section{Markov Decision Process}
We assume the underlying problem is defined as Markov Decision Process (MDP). An MDP is defined as a tuple $\mathcal{M} = (\mathcal{S},\mathcal{A},\mathcal{P},\mathcal{H},\rho_0,\mathcal{R})$, where $\mathcal{S}$ is a finite set of states; $\mathcal{A}$ set of actions. $\mathcal{P} = P(s,a,s')$ is a Markovian state transition model, the probability of making transition to $s'$ upon taking action $a$ in state $s$. $\mathcal{H}$ is solution horizon of MDP, so that MDP terminates after $\mathcal{H}$ steps. $\rho_0$ is distribution over which initial states are chosen and $\mathcal{R}:\mathcal{S} \times \mathcal{A} \to {\rm I\!R}$ is reward function measuring the performance of agent and is assumed to be bounded by $R_{max}$. Total return for all states $s_i \in \rho_0$ is defined as sum of discounted reward $J = \sum_{i=t}^\mathcal{H}\gamma^{i-t}R(s_i,a_i)$, $\gamma \in [0,1)$ being the discount factor. A policy $\pi: \mathcal{S} \to \mathcal{A}$ is a mapping from states $\mathcal{S}$ to a probability distribution over set of actions in $\mathcal{A}$. The agent's goal is to find a policy $\pi^\star$ which maximize the total return.

We formalize the underlying transfer problem by considering a source and target MDP $\mathcal{M}^S = (\mathcal{S},\mathcal{A},\mathcal{P},\mathcal{H},\rho_0,\mathcal{R})^S$,  $\mathcal{M}^T = (\mathcal{S},\mathcal{A},\mathcal{P},\mathcal{H},\rho_0,\mathcal{R})^T$, with its own state space, action space and transition model respectively. In general the state space can be completely different in two domains. Regarding the action space of two domains we assume,\\
\textit{\textbf{Assumption 1:}} The cardinality of the discrete action space is same in source and target task
\begin{equation}
|\mathcal{A}^{(T)}| = |\mathcal{A}^{(S)}|
\label{action_cardinality}
\end{equation}
but the limits on action amplitude can be different
\begin{equation}
||\mathcal{A}^{(S)}||_2 \leq \tau^{(S)} ,
||\mathcal{A}^{(T)}||_2 \leq \tau^{(T)} 
\end{equation}
We assume an invertible mapping $\chi_s$ provides correspondence between the two state space of source and target model. We will use $\hat s^S_i$, $\hat s^T_i$ to denote the corresponding projected states at time $``i"$ from target and source spaces respectively.

The transition probabilities $\mathcal{P}^S,\mathcal{P}^T$ also differ. However, we assume the physics of the problem share some similarities in the underlying principles.

\textit{\textbf{Assumption 2:}} The transition model for the source task is available or that we can sample from a source model simulator. This assumption is not very restrictive since the designer can always select/create related source task for given target task. 
\begin{equation}
{s}^{S}_{i+1} = \mathcal{P}^{(S)}(\hat{s}^{S}_{i},a^{S}_i)
\end{equation}
The target transition probabilities $\mathcal{P}^{(T)}$ is modeled online using state-action-state $(s_t,a_t,s_{t+1})$ triplets collected along the trajectories generated by some random exploration policy. We call the approximate model as the apprentice to target $\hat \mathcal{P}^{(T)}$.
\subsection{Algorithm: TA-TL}
For every initial condition in target task $s^{T}_0 \in \mathcal{S}^{(T)}$; $s^{T}_0$ are mapped to source space to find the corresponding initial condition of source task. 
\begin{equation}
\hat{s}^{S}_i = {\chi_s}^+(s^{T}_i)
\end{equation}     
where ${\chi_s}^+$ is the inverse mapping from target to source and $\hat{s}^{S}_{i}$ represents the image of $s^{T}_i$ in source state space. For the mapped state in source task, a greedy action is selected using learned $Q^{(S)}(\hat{s}^{S}_{i},a^S_i)$ state-action value function.
\begin{equation}
a^{S}_i \leftarrow arg\max_{a^S_i}(Q^{(S)}(\hat{s}^{S}_{i},a^S_i))
\end{equation}

Using selected action $a^{S}_i$ the source model at state $\hat{s}^{(S)}_{i}$ is propagated to ${s}^{S}_{i+1}$.
The propagated state in source task is mapped back to the target space using inter-task mapping function, 
\begin{equation}
\hat{s}^{T}_{i+1} = {\chi_s}(s^{S}_{i+1})
\end{equation} 
where $\hat{s}^{T}_{i+1}$ is the image of $s^{S}_{i+1}$ in target space. From Assumption-1, every selected action in source task has greedy correspondence in target task. Using this equivalence of actions, for every selected action in source task an equivalent action in target task is selected as $a^{T}_i \in \mathcal{A}^{(T)}$. The selected action for target task is augmented with $a^{(T)}_{ad} \in \mathcal{A}^{(T)}_{ad}$ derived from adaptive policy,
\begin{eqnarray}
\pi^{(T)}_{ad} &=& {\mathcal{P}}^{(T)}_{(S)}(s^S_i,a^S_i)-\hat{\mathcal{P}}^{(T)}(s^T_i,a^T_i) \\
a^{(T)}_{ad} &=& \pi^{(T)}_{ad}(s^{T})
\label{eq:adaptive_term}
\end{eqnarray}
where $\hat{\mathcal{P}}^{(T)}(s^T_i,a^T_i)$ is apprentice model and
\begin{equation}
{\mathcal{P}}^{(T)}_{(S)}(s^S_i,a^S_i) ={\chi_s}({\mathcal{P}}^{(S)}(s^S_i,a^S_i))
\label{eq:projected_model}
\end{equation}
is the projected source model on to target space. The set 
$\mathcal{A}^{(T)}_{ad}$ is adaptive action space such that $|\mathcal{A}^{(T)}_{ad}| \geq |\mathcal{A}^{(T)}|$ and $||\mathcal{A}^{(T)}_{ad}|| \leq ||\pi^{(T)}_{ad}||_{\infty}$

The total transferred policy for solving a related target task is proposed to be a linear combination of mapped optimal policy and an adaptive policy as follows, 
\begin{equation}
\pi^{*(T)}(.) = \pi^{*(S)}({\chi_s}(.)) + \mathcal{K}\pi^{(T)}_{ad}(.)
\label{eq:modified_optimal_policy}
\end{equation} 
\subsection{Analysis}
\begin{theorem}
    For any given small  $\epsilon \geq 0$, there exists a $\delta(\epsilon)$ such that the difference between true target model and target apprentice model over the entire state-action space be
	$$\|\mathcal{P}^{(T)}(s,a) - \hat{\mathcal{P}}^{(T)}(s,a)\| \leq \delta(\epsilon),\hspace{3mm}  \forall (s,a) \in \mathcal{S} \times \mathcal{A} $$
Then using $a^{S} = \pi^{*(S)}(.)$, the optimal policy for source task,
the modified policy (\ref{eq:modified_optimal_policy}) can be shown to be $\epsilon$-optimal in the target task
\end{theorem}
\begin{proof}
	We analyze the admissibility of the augmented policy for the target space. Target model is assumed to be any nonlinear, control affine system and source model be any nonlinear system. The discrete transition model for both source and target model can be considered as follows,
	\begin{eqnarray}
	\mathcal{P}^{(S)}(s,a):s^{S}_{i+1} &=& \mathcal{F}^{*(S)}(s^{S}_i,a^{S}_i)\\
    \label{eq:True_Source_model}
	\mathcal{P}^{(T)}(s,a):s^{T}_{i+1} &=& \mathcal{F}^{*(T)}(s^{T}_i) + \mathcal{B}a^{T}_i
    \label{eq:True_Target_model}
	\end{eqnarray} 
	where $s^{S} \in \mathcal{S}^{(S)},s^{T} \in \mathcal{S}^{(T)}$ and $a^{S} \in \mathcal{A}^{(S)},a^{T} \in \mathcal{A}^{(T)}$.

The target apprentice is an approximation to the target model. We retain the control affine property of the target model by using appropriate basis of the single layer neural network, to model the target dynamics. The approximate or the apprentice model of the target can be written as function of network weights and basis as,
\begin{eqnarray}
	\hat \mathcal{P}^{(T)}(s,a):s^{T}_{i+1} &=& \hat \mathcal{F}^{(T)}(s_i) + \hat \mathcal{B}a^{(T)}_i\\
    &=& \hat W \phi(s^{(T)}_i) + \hat \mathcal{B}a^{(T)}_i\\
    &=& \left[\hat W  \hspace{2mm} \hat \mathcal{B}\right]\times\left[\phi(s^{(T)}_i) \hspace{2mm}  a^{(T)}_i\right]^T
	\label{eq:Apprentice_model}
	\end{eqnarray}
    where $\hat \mathcal{F}^{(T)}(s_t) = \hat W^{(T)}\phi(s^{(T)}_i)$ and $\left[\hat W  \hspace{2mm} \hat \mathcal{B}\right]$, $\psi(s^{T}_i,a^{T}_i) = \left[\phi(s^{(T)}_i) \hspace{2mm}  a^{(T)}_i\right]$ be target apprentice network weights and basis function.

Sampling the action from modified target optimal policy (\ref{eq:modified_optimal_policy}) $ a^{(T)}_i = \pi^{*(T)}((s_i^T))$ and applying it to target model following holds,
	\begin{equation}
	s^{(T)}_{i+1} = \mathcal{F}^{*(T)}(s^{T}_i)  + \mathcal{B}a^{(T)}_{S,i} + \mathcal{B}\mathcal{K} a^{(T)}_{ad,i}
	\end{equation}
where $a^{(T)}_{S,i} = \pi^{*(S)}({\chi_s}(s^{T}_i))$ is the mapped optimal action to target space corresponding to source optimal policy and $a^{(T)}_{ad,i} = \pi^{(T)}_{ad}(s^{T}_i)$ is modification term to mapped optimal action to cancel the effects of model error.

From definition of model adaptive policy (\ref{eq:adaptive_term}) and apprentice model (\ref{eq:Apprentice_model}), above expression can be simplified as
	\begin{eqnarray}
	s^{(T)}_{i+1} &=& \mathcal{F}^{*(T)}(s^{T}_i) + \mathcal{B}a^{(T)}_{(S),i}\nonumber\\&& + \mathcal{B}\mathcal{K}\left( {\mathcal{P}}^{(T)}_{(S)}(s^S_i,a^S_i) - \hat{\mathcal{P}}^{(T)}(s^T_i,a^T_i)\right)\\
	s^{(T)}_{i+1} &=& \mathcal{F}^{*(T)}(s^{T}_i) + \mathcal{B}a^{(T)}_{(S),i}\nonumber\\&& + \mathcal{B}\mathcal{K}\left({\mathcal{P}}^{(T)}_{(S)}(s^S_i,a^S_i) - \hat \mathcal{F}^{(T)}(s_i) - \hat \mathcal{B}a^{(T)}_i\right)\nonumber
	\end{eqnarray}
For choice of policy mixture coefficient $\mathcal{K} = 1/\hat \mathcal{B}$.	The above expression Simplifies to,
	\begin{equation}
	s^{T}_{i+1} = \alpha{\mathcal{P}}^{(T)}_{(S)}(s^S_i,a^S_i) + \left(\mathcal{P}^{(T)}(s,a) - \alpha \hat{\mathcal{P}}^{(T)}(s,a)\right)
	\label{eq:Adapted_target_model}
	\end{equation} 
Where $\alpha = \mathcal{B}/\hat{\mathcal{B}}$, and for persistently exciting data collected for apprentice model learning $\mathcal{D} = [s_i,a_i,s_{i+1}]_{i = 1}^N$ convergence of the parameters to true value can be shown \cite{liu2013convergence}, ensuring $\alpha \approx 1$
 
Using (\ref{eq:projected_model}) and (\ref{eq:True_Source_model}) the above expression (\ref{eq:Adapted_target_model}), can be rewritten in terms of source transition model using the inter-task mapping function $\mathbf{\chi_S}$ as
	\begin{equation}
	s^{T}_{i+1} = \mathbf{\chi_s}\left(\mathcal{F}^{*(S)}(s^{S}_i,a^{S}_i)\right) + \epsilon
	\label{eq:Source_model_equation}
	\end{equation}
    where $\epsilon = \mathbf{\chi_s}\|\mathcal{P}^{(T)}(s,a) - \alpha \hat{\mathcal{P}}^{(T)}(s,a)\|$
    
	Expression (\ref{eq:Source_model_equation}) demonstrates that implementation of the modified optimal policy (\ref{eq:modified_optimal_policy})  in target task is equivalent to projecting the source optimal trajectories on to the target space. Assuming existence of unique correspondence between source to target task space, we prove the  policy (\ref{eq:modified_optimal_policy}) leads to $\epsilon$-optimal solution in target model.
\end{proof}
\section{Target Task Apprentice Learning}
Target apprentice is an approximate model for target task. In this paper we consider the target model apprenticeship learning using any random policy which explores randomly the target domain \cite{abbeel2005exploration}.
We re-use the dataset, state-action-state triplets generated through random policy for manifold alignment, for target apprentice learning. This data reuse leads to further saving of time and processing for sample generation for apprentice learning. 
\subsection{Apprentice Learning: Algorithm}
Using any random policy, $\pi$ we explore the target model to collect state-action-state triplets to learn the target apprentice to enable transfer learning using apprentice model.
\begin{enumerate}
	\item Run $k$ trials in target task under the random $\pi$ policy for $N_T$ steps. Save the state trajectories experienced. 
    \item Using accumulated data of state-action-state triplets, estimate the system dynamics using least square linear regression for linearly parametrized model and store the system parameters $\theta = [\hat W,\hat \mathcal{B}]$.
    \item Evaluate the utility of the projected policy $\hat{\pi}^{(T)} = \pi^{*(S)}(\mathbf{\chi_s}(s^{T}))$ in target model on both true and approximate system, $\mathcal{M}^{(T)}$ and $\hat{\mathcal{M}}^{(T)}$. Utility function $U_M(\pi)$ is defined as average reward accumulated for $k$ trials.
    \item If $U_M(\hat{\pi}^{(T)}) - U_{\hat M}(\hat{\pi}^{(T)}) \leq \zeta$,  return $\theta$, where $\zeta$ is some chosen small threshold.
\end{enumerate}

\section{Experiments \& Results}
We present results from five experiments to evaluate the proposed transfer learning framework. The first two experiments consider transfer learning in same domains but with different transition models and action spaces. The first problem is in discrete state and action space, while the second problem is of continuous state space and non-stationary transition model in the target task. The third and fourth experiment focuses on cross-domain transfer where the policy from the cart-pole, mountain car is transferred to the bicycle problem, inverted pendulum domains respectively. We also demonstrate the presented approach being robust to negative transfer through our final experiment. We compare the presented Target Apprentice TL (TA-TL) against existing state of the art Transfer in RL, Unsupervised Manifold Alignment (UMA-TL) \cite{ammar2015unsupervised} and no transfer RL (Fitted Q-Learning).
\begin{algorithm}
	\caption{Apprentice Learning}
	\label{alg:Apprentice}
	\begin{algorithmic}[1]
		\STATE {\bfseries Input:} A Random policy $\pi$, trials $N_T$, convergence criterion $\zeta$
		\REPEAT
		\STATE Initialize $(s_0,a_0) \in \mathcal{S} \times \mathcal{A}$.
		\FOR{$i=1$ {\bfseries to} $N_T$}
		\STATE Execute the random policy $\pi$ in target model
		\STATE $$a_i = \pi(S_i)$$
		\STATE Propagate the target model at state $s_i$ under action $a_i$ $$s_{i+1} = \mathcal{P}^{(T)}(s_i,a_i)$$
		\STATE Save state trajectories $\mathcal{D}_{N_T}(i,:)=(s_i,a_i,s_{i+1})$
		\ENDFOR
		\STATE Solve the least square linear regression problem on training data set $\mathcal{D}_{N_T}$
		$$\min_{\theta}\sum_{i = 1}^{N_T}(s_{i+1} - \theta'\psi(s_i,a_i))$$
		\STATE Evaluate the utility of policy $\pi^{(S)}(\mathbf{\chi}_s(s))$ in true model $\mathcal{M}^{(T)}$ and in approximated model $\hat{\mathcal{M}}^{(T)}$
		\UNTIL{$U_M(\pi^S) - U_{\hat M}(\pi^S) \leq \zeta$} return  $\theta$
		\end{algorithmic}
\end{algorithm}
\subsection{Same-Domain Transfer}
We learn the optimal policy in the source task using FQI. In each problem, distinction in the environment/system parameters makes the source and target tasks different. The target and source domains have the same state-space but different transition models and action spaces. We also do not need target reward model be similar to source task, as the proposed algorithm directly adapts the policy from the source task and does not engage in RL in the target domain. 

\subsubsection{Grid World to Windy Grid World}
 The source task in this experiment is Non-Windy (NW) grid world. The state variables describing the system are grid positions. The RL objective is to navigate an agent through obstacles from start to goal position optimally. The  admissible actions are up $(+1)$, down $(-1)$, right $(+1)$ and left $(-1)$. The reward function is $+10$ for reaching goal position, $-1$ for hitting obstacles and $0$ everywhere else. The target domain is same as the source but with the added wind which affects the transition model in parts of the state-space (see Figure \ref{fig:tlwithmodellearning}).
 
 The optimal policy in source task (non-windy grid world) $\pi^{*(S)}$ is learned using Q-Iteration. We do not need any inter-task mapping as the source, and target state spaces are identical. We start with $100$ randomly sampled starting position and execute policy $\pi^{*(S)}$ in the target domain and collect samples for apprentice model learning. Empirically, we show the proposed method (TA-TL) provides a sample efficient TL algorithm compared to other transfer techniques.
 Figure \ref{fig:sourcetaskrlnowind} and \ref{fig:tlwithmodellearning} shows the results of same domain transfer in the grid world, demonstrating TA-TL achieving successful transfer in navigating through the grid with obstacles and wind bias. Figure \ref{fig:avg_reward_Windy} and \ref{fig:solutiontime} shows the quality of transfer through faster convergence to average maximum reward with lesser training samples compared to UMA-TL and RL methods.  The presented algorithm can attain maximum average reward in reaching goal position in $ \sim 2 \times 10^4$ steps. UMA-TL and RL algorithm achieve similar performance in $ \sim 1.2 \times 10^5$ and $ \sim 1.7 \times 10^5$ steps respectively, nearly one order higher compared to proposed TA-TL.
 
\begin{figure*}
    \centering
    \begin{subfigure}{0.5\columnwidth}
        \includegraphics[width=0.8\textwidth]{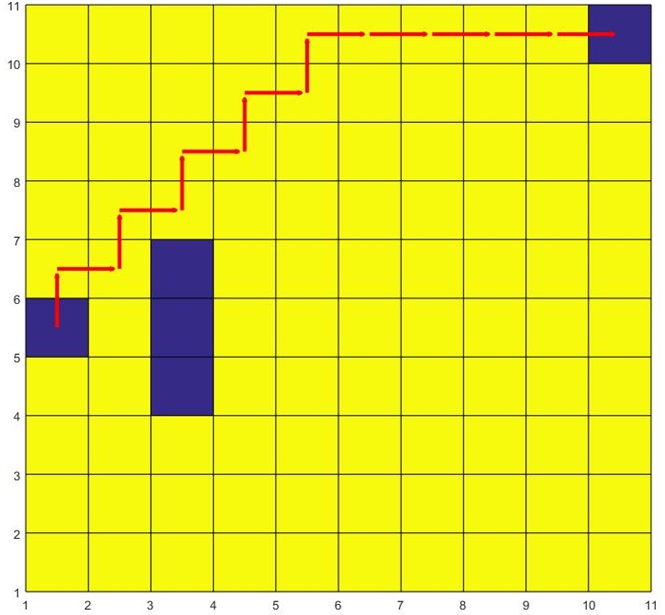}
        \caption{}
        \label{fig:sourcetaskrlnowind}
    \end{subfigure}
    \begin{subfigure}{0.5\columnwidth}
        \includegraphics[width=0.8\textwidth]{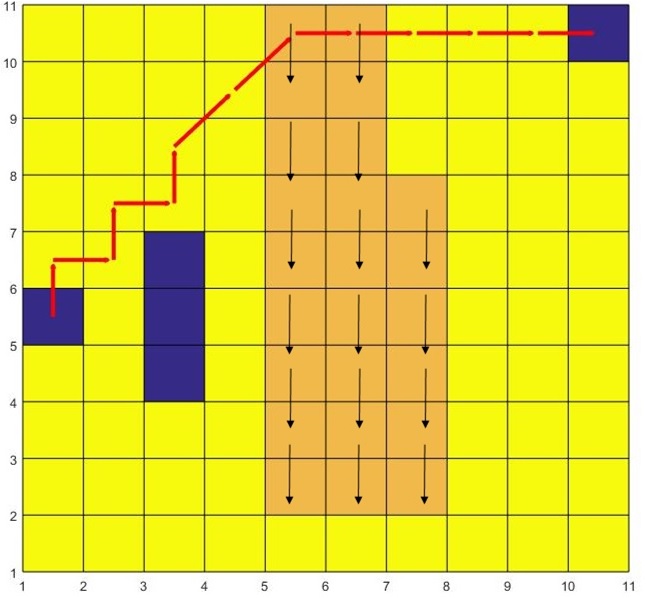}
        \caption{}
        \label{fig:tlwithmodellearning}
    \end{subfigure}
    \begin{subfigure}{0.5\columnwidth}
        \includegraphics[width=\textwidth]{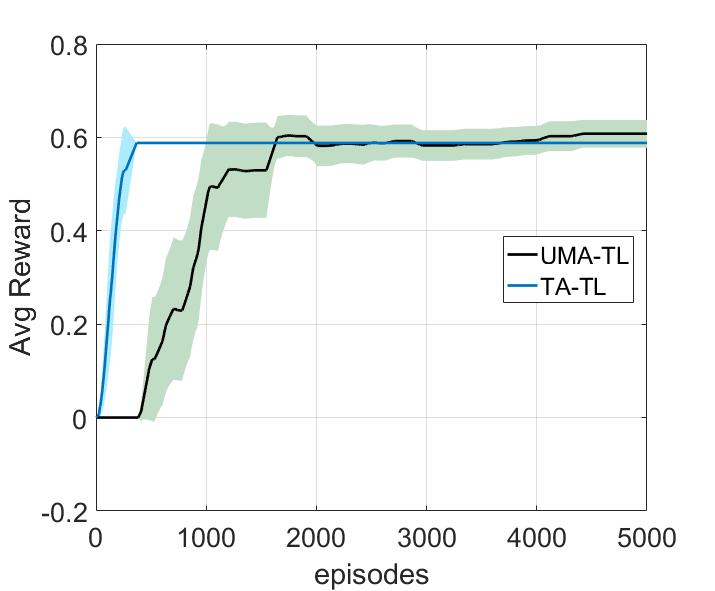}
        \caption{}
        \label{fig:avg_reward_Windy}
    \end{subfigure}
    \begin{subfigure}{0.5\columnwidth}
        \includegraphics[width=\textwidth]{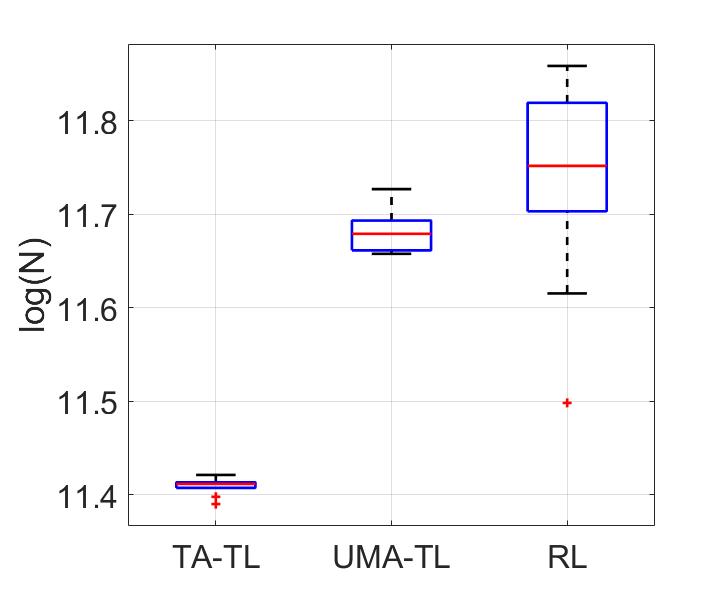}
        \vspace{-0.2in}
        \caption{}
        \label{fig:solutiontime}
    \end{subfigure}
    \caption{Non windy to Windy grid World Transfer:(a) \& (b) Agent navigating through grid world in source and target domain (c) Average Rewards \& (d) Training length, Comparing quality of transfer for TA-TL and UMA-TL through convergence rate of Average Reward and Training Length}
    \label{fig:grid_world}
\end{figure*}

\begin{figure*}[tbh]
    \centering
    \begin{subfigure}{0.5\columnwidth}
        \includegraphics[width=\textwidth]{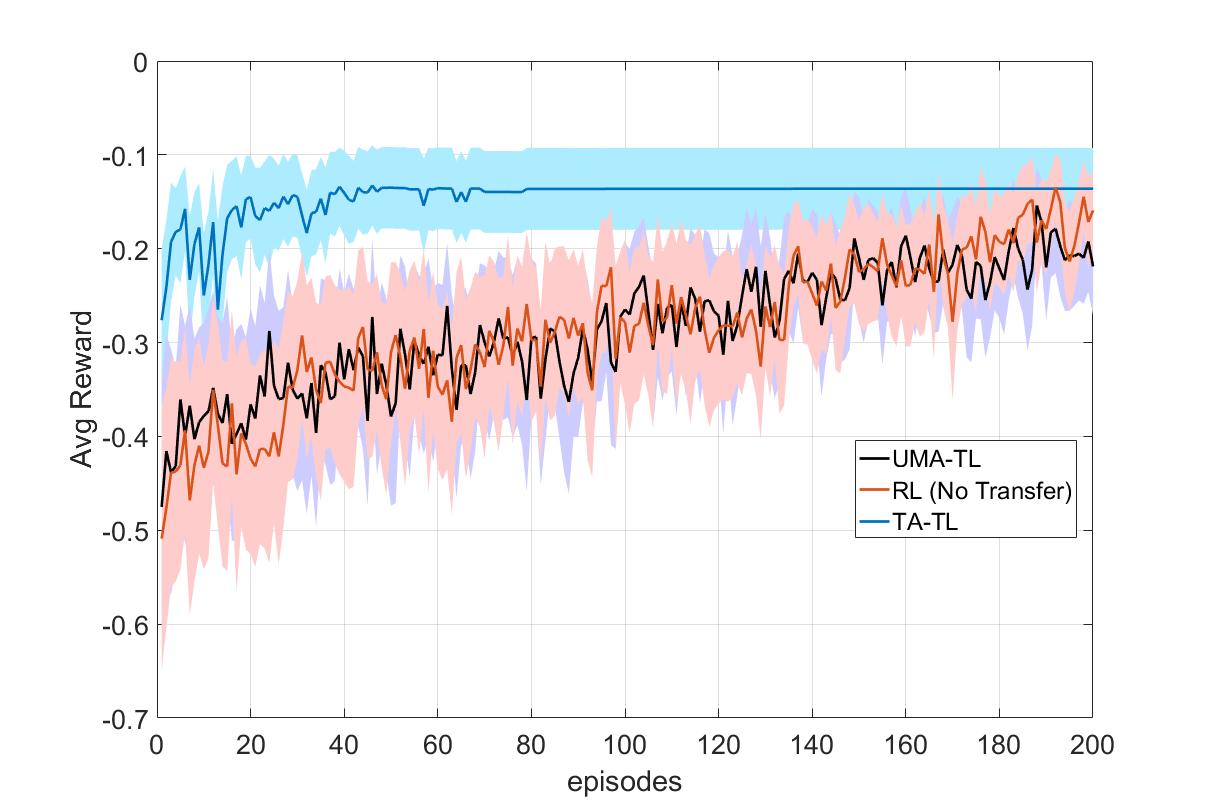}
        \caption{}
        \label{fig:avg_reward_IP}
    \end{subfigure}
    \begin{subfigure}{0.5\columnwidth}
        \includegraphics[width=\textwidth]{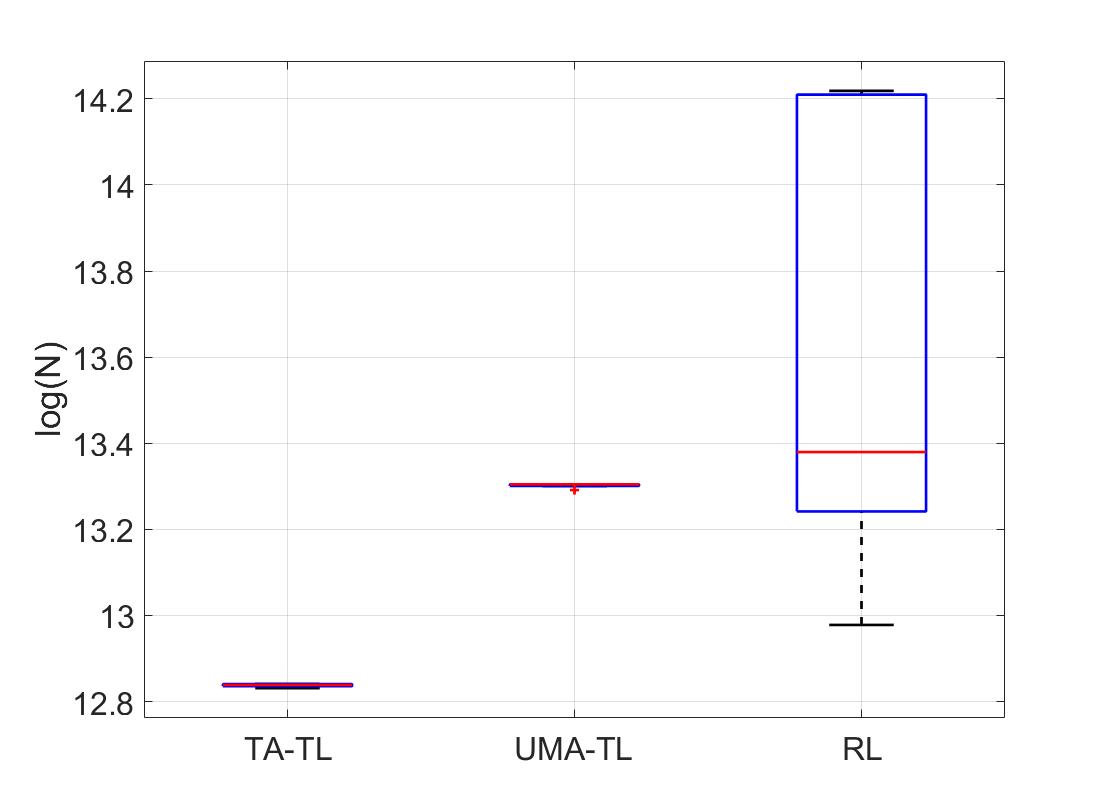}
        \caption{}
        \label{fig:solutiontime_IP}
    \end{subfigure}
    \begin{subfigure}{0.5\columnwidth}
            \includegraphics[width=\textwidth]{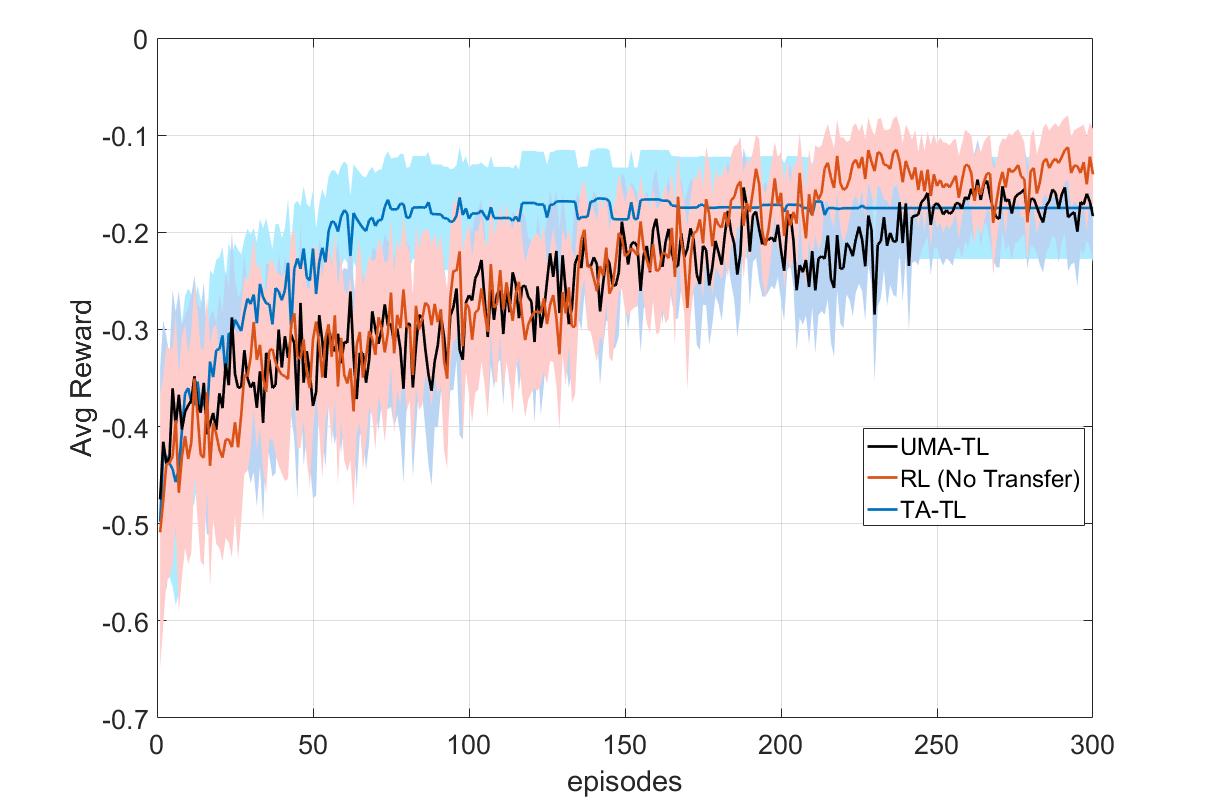}
            \caption{}
            \label{fig:avg_reward_MC}
    \end{subfigure}
    \begin{subfigure}{0.5\columnwidth}
            \includegraphics[width=\textwidth]{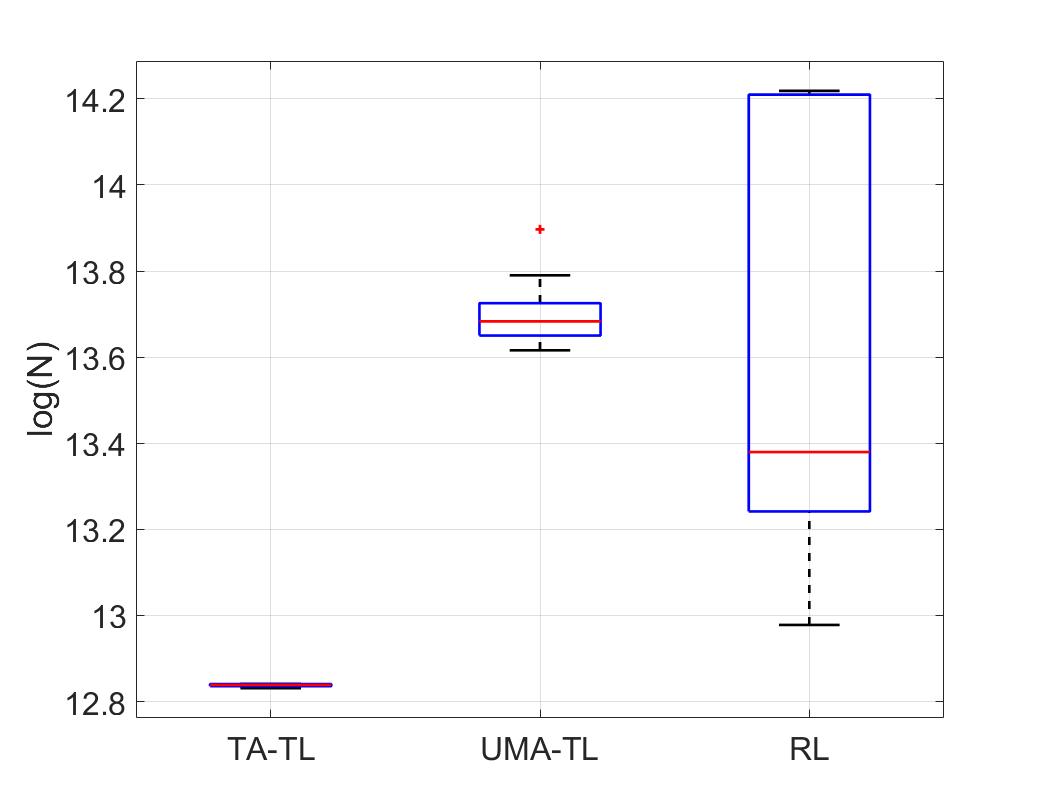}
            \caption{}
            \label{fig:solutiontime_MC}
    \end{subfigure}
    \caption{IP to Non-stationary IP Transfer: (a) Average Rewards and (b) Training length, MC to IP Transfer: (c) Average Rewards and (d) Training length}
    \label{fig:IP_IP}
\end{figure*}

\begin{figure*}[tbh]
    \centering
    \begin{subfigure}{0.5\columnwidth}
        \includegraphics[width=\textwidth]{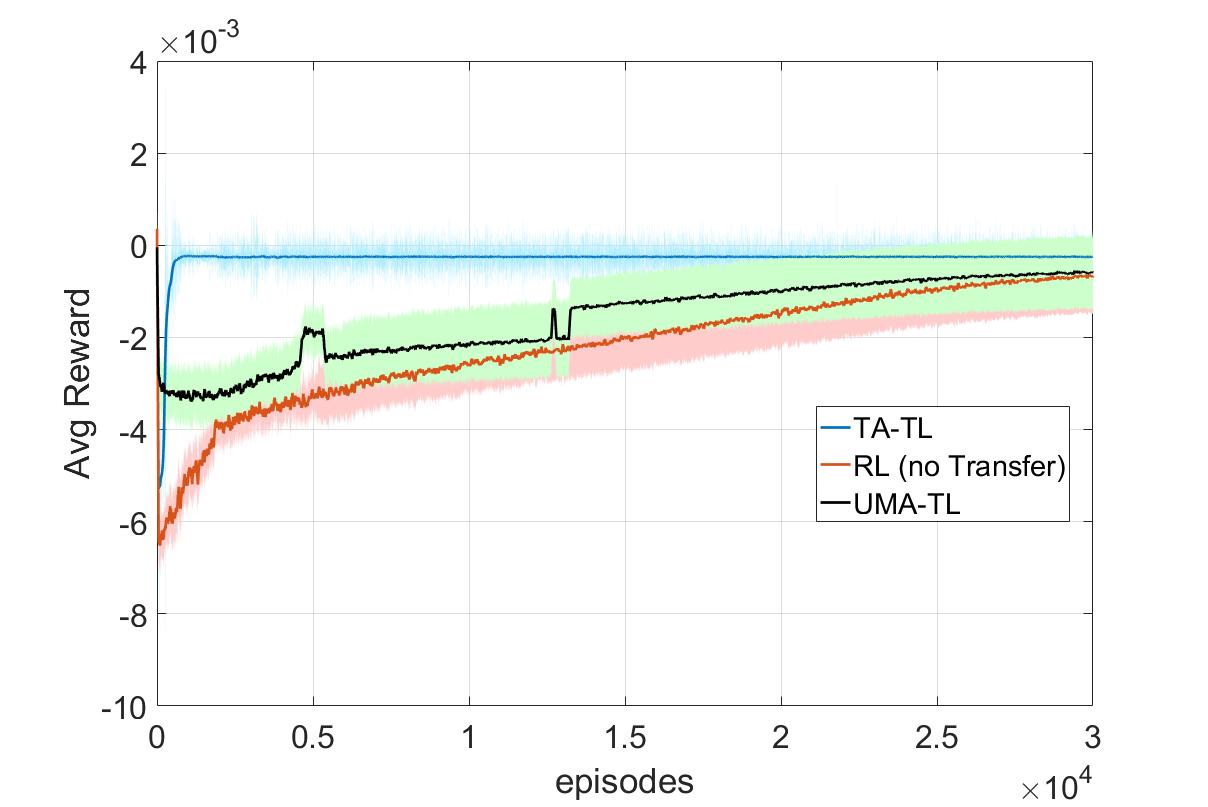}
        \caption{}
        \label{fig:avg_reward_Bike}
    \end{subfigure}
    \begin{subfigure}{0.5\columnwidth}
        \includegraphics[width=\textwidth]{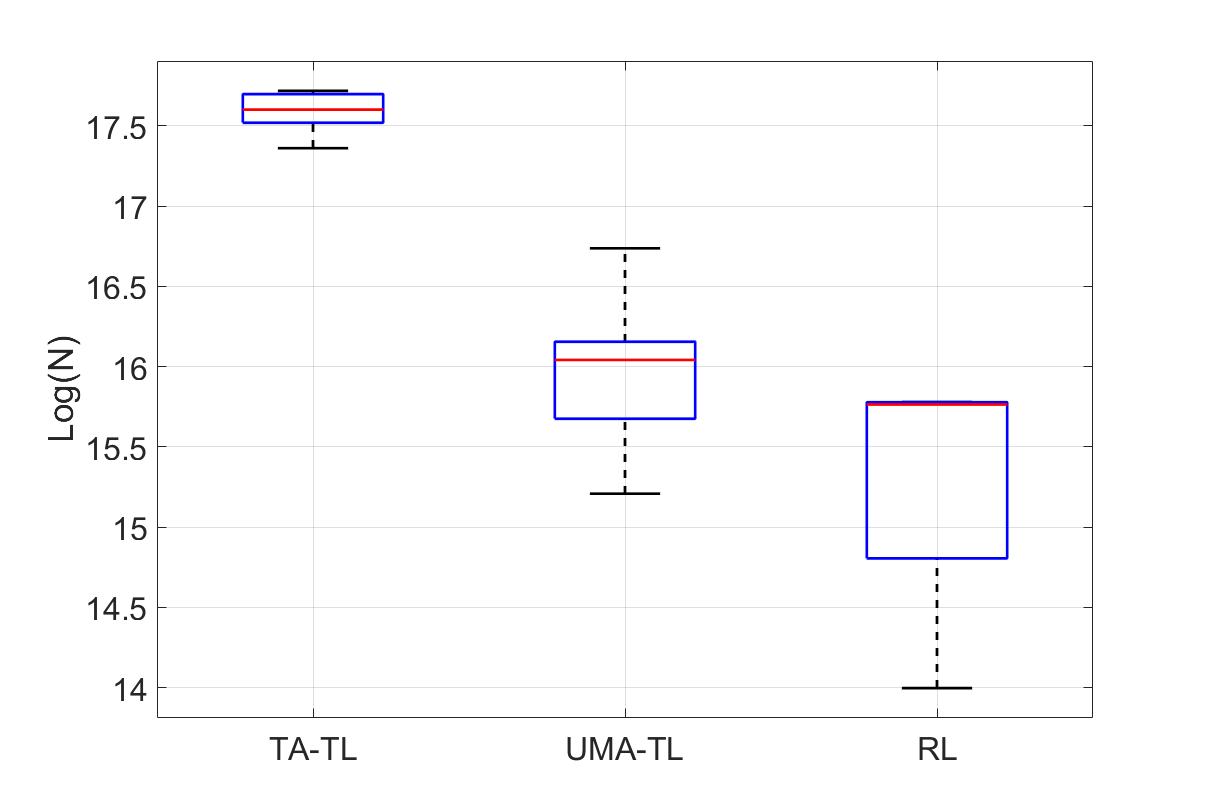}
        \caption{}
        \label{fig:solutiontime_Bike}
    \end{subfigure}
    \begin{subfigure}{0.5\columnwidth}
        \includegraphics[width=\textwidth]{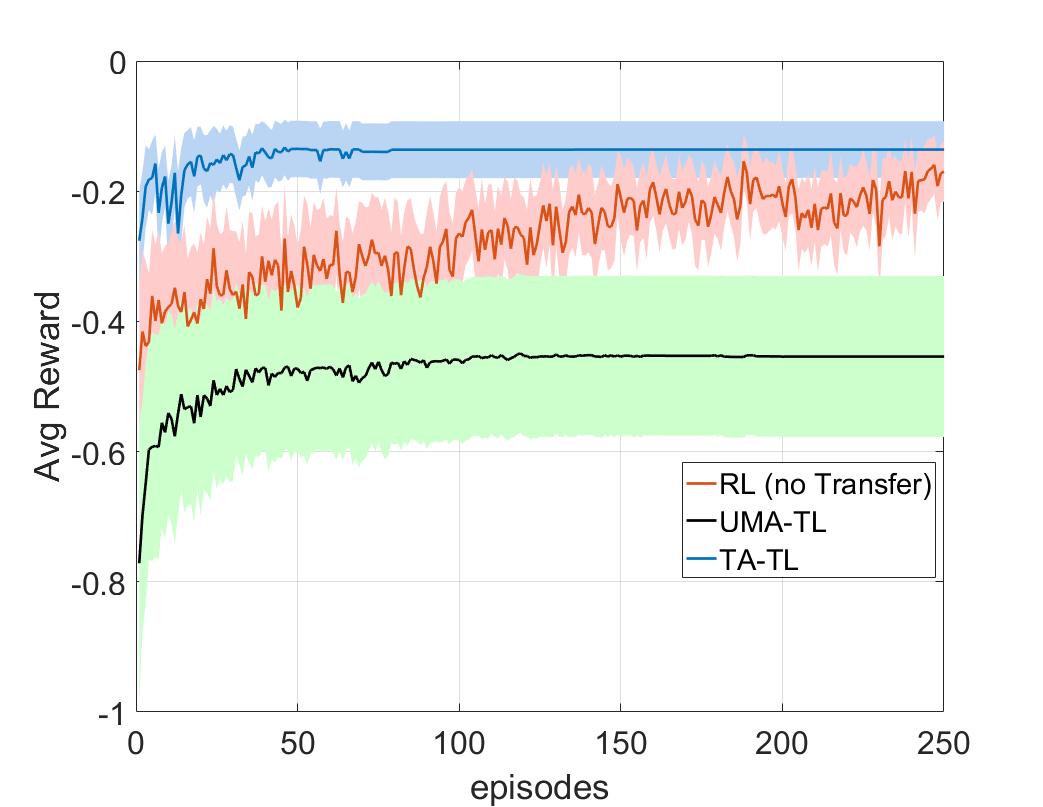}
        \caption{}
        \label{fig:avg_reward_NT}
    \end{subfigure}
    \begin{subfigure}{0.5\columnwidth}
        \includegraphics[width=\textwidth]{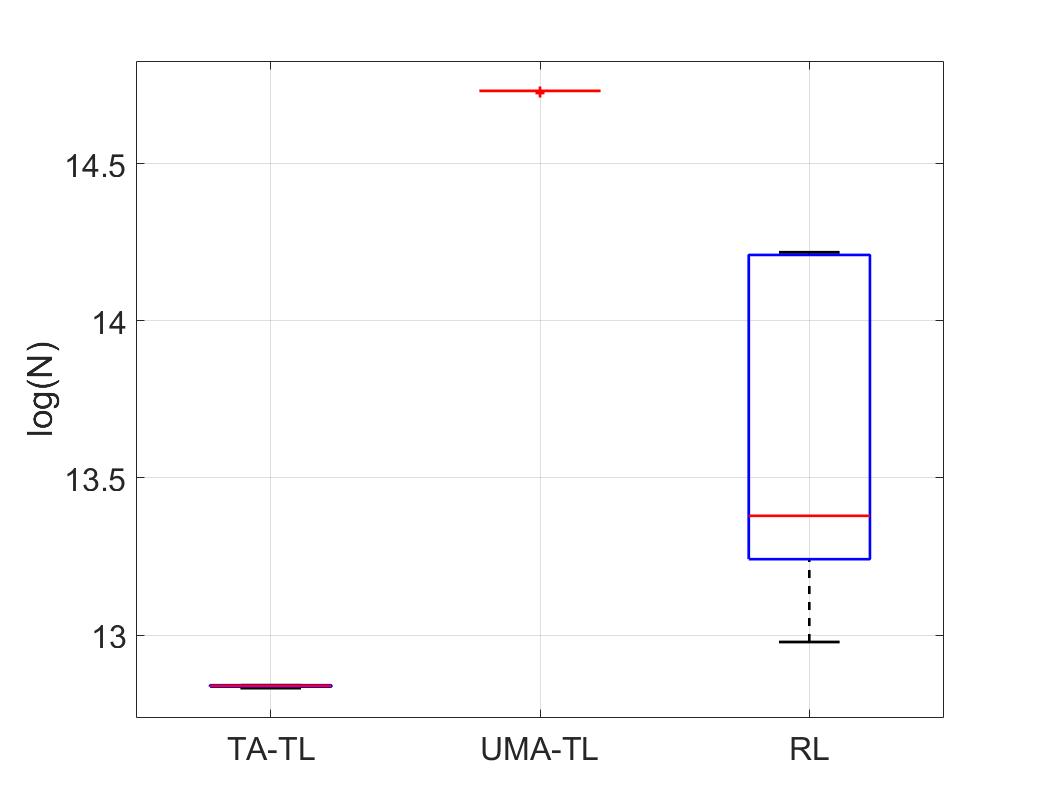}
        \caption{}
        \label{fig:solutiontime_NT}
    \end{subfigure}
    \caption{Transfer Cart-Pole to Bike Balancing task: (a) Average Rewards and (b) Total simulation time (seconds) the agent can balance the bike, Negative Transfer Inverted Pendulum: (c) Average Reward and (d) Training length}
    \label{fig:IP_IP}
\end{figure*}

\subsubsection{Inverted Pendulum (IP) to time-varying IP}
We demonstrate our approach for a continuous state domain, Inverted Pendulum (IP) swing-up and balance. The source task is the conventional IP domain \cite{sutton1998reinforcement}. The target task is a non-stationary inverted pendulum, where the length and mass of the pendulum are continuously time varying with function $L_i = L_0 + 0.5cos(\frac{\pi i}{50})$ and $M_i = M_0 + 0.5cos(\frac{\pi i}{50})$, where $L_0 = 1$, $M_0 = 1$ and $i = 1 \ldots N$. The state variables describing the system are angle and angular velocity $\{\theta, \dot \theta\}$ where $\theta, \dot \theta \in [-\pi,\pi]$. The RL objective is to swing-up and balance the pendulum upright such that $\theta = 0, \dot \theta = 0$.  The reward function is selected as $r(\theta,\dot \theta) = -10|\theta|^2 - 5|\dot \theta|^2$, which yields maximum value at upright position and minimum at the down-most position. The action space is: full throttle right $(+\tau)$, full throttle left $(-\tau)$ and zero throttle. Note that the domain is tricky, since full throttle action is assumed to not generate enough torque to be able to lift the pendulum to the upright position, hence, the agent must learn to swing the pendulum so that it generates oscillations and leverages angular momentum to go to the upright position. The target task differs from the source task in the transition model.

The source task use FQI learning with single layer Radial Basis Functions (RBF) network. The Q function is modeled as linear combination of weights and basis as $Q = w'\xi(s)$. We use RBF bases ``$\xi(s)$'' for value function approximation with bandwidth $\sigma = 1.2I$ and $20$ centers spanning space $x,\dot x \in [-\pi, \pi]$ with network learning rate of $\Gamma = 5 \times 10^{-3}$ for FQI iterations. 

Figure \ref{fig:avg_reward_IP} and \ref{fig:solutiontime_IP} shows the quality of transfer through faster convergence to average maximum reward with lesser training samples for proposed TA-TL method compared to UMA-TL and RL methods.
\subsection{Cross-Domain Transfer}
Next, we consider an even more challenging transfer setting: cross-domain transfer. The problem setup is similar to same domain transfer with the notable distinction being the state spaces are different for the source and target tasks. 

\subsubsection{Cart-Pole to Bicycle}
Our main result is the task where an agent learns to ride the bicycle. We consider the problem of learning to balance; we do not concentrate on the navigation problem to some goal position. Since RL for navigation is more of trajectory optimization problem, i.e., the agent learn to focus on maneuvering towards the target once it has learned to balance the bicycle upright. Balancing is a more interesting problem when the bicycle is below critical velocity. Usually critically velocity $V_c$, above which bicycle is self stabilizing is approximately $V_c = 4m/s$ to $5m/s$. We are trying to learn to balance a unstable bicycle with forward velocity $V < V_c$ i.e. $V = 2.778m/s$. We also simulate the imperfect balance by inducing random noise in the CG displacement of rider, by up to $2cm$ from zero position.

At every time step agent receives information about the state of the bicycle, the angle and angular velocity of the handlebar and the bike from vertical $(\theta, \dot \theta, \omega, \dot \omega)$ respectively. For the given state the agent is in, it chooses an action of applying torque to handlebar, $T \in [-2Nm,0,2Nm]$ trying to keep bike upright. The details of bicycle dynamics are beyond the scope of this paper, interested readers are referred to \cite{randlov1998learning,aastrom2005bicycle} and references therein.

We use the Cart-Pole as source task for learning to balance bicycle. The bicycle balance problem is not so different from the cart pole: in both the cases, the objective is to keep the unstable system upright. The objective of balance is achieved in both the systems by moving in the direction of fall. However, the control in the cart pole affects more directly the angle of the pole, i.e., move the cart such that it is always under the pole. In the bicycle, the control is to move the handlebar in the direction of fall. However, balancing the bike is not so simple, to turn the bike under itself, one must first steer in the other direction, this is called counter steering \cite{aastrom2005bicycle}. We observe that both cart pole and bicycle has this commonality in dynamical behavior, as both the system have a non-minimum phase that is the presence of unstable zero. They tend to move initially in the direction opposite to the applied control. This similarity qualifies the cart-pole system as an appropriate choice of source model for bicycle balance task.

Cart pole is characterized by state vector $[x, \dot x, \theta, \dot \theta]$, i.e., position, the velocity of cart and angle, angular velocity of the pendulum. The action space is the force applied to cart $F \in [-20N,0,20N]$. Cross-domain transfer requires a correspondence between inter-task space manifold for mapping the learned policy and source transition model from source to target space and back. We use UMA to discover the correspondence between state space of bicycle and cart pole model. We use FQI to solve for optimal policy in the source, cart pole model. A linear network with cart pole states $\xi(s) = [x, \dot x, \theta, \dot \theta]$ is used as basis vector to approximate the action value function, with network learning rate $\Gamma = 1 \times 10^{-4}$.

Figure \ref{fig:avg_reward_Bike} shows the average reward accumulated by TA-TL, UMA-TL and RL (no transfer) in learning to balance the bicycle. In the typical learning process, TA-TL out performs other transfer method and converges to maximum average reward in 1700 episodes. Each episode is a simulation run until the policy can balance bicycle without toppling; the episode ends if the bicycle falls or max time of 1000s is reached. Figure \ref{fig:solutiontime_Bike} shows total time the bicycle was balanced upright by each method. The bike balancing time for TA-TL methods is highest and is $\approx 40$ times more compared to UMA-TL method.

\subsubsection{Mountain Car (MC) to Inverted Pendulum (IP)}
 We have tested the cross-domain transfer between mountain car to an inverted pendulum. The source and target task are characterized by different state and action space. The source task MC is a benchmark RL problem of driving an under-powered car up a hill. The dynamics of MC are described by two continues state variables $(x,\dot x)$ where $x \in [-1.2 , 0.6]$ and $\dot x \in [-0.07 , 0.07]$. The input action takes three distinct values $(+1)$ full throttle right, $(-1)$ full throttle left and $(0)$ no throttle. The reward function is proportional to negative of the square distance of the car from goal position. The target task is conventional IP as described in the previous experiment. 
 
We utilize UMA to obtain this mapping as described in Section \ref{MA}. We do not report the training time to learn the intertask mapping since it is common to both TA-TL and UMA-TL methods. We used a random policy to generate samples for manifold alignment and for target apprentice learning. The source task uses FQI learning with single layer RBF network for optimal policy generation. The source Q-function is modeled using function approximation as $Q = w'\xi(s)$ using RBF as bases ``$\xi(s)$'' with bandwidth $\sigma = diag[0.3, 0.1]$ and $20$ centers spanning space $ x \in [-1.2, 0.6]$ and $ \dot x \in [-0.07, 0.07]$ with network learning rate of $\Gamma = 0.15 \times 10^{-3}$. 
For all above results the training length involved with TA-TL method in Figure \ref{fig:solutiontime_MC}, \ref{fig:solutiontime_IP}, and \ref{fig:solutiontime} is sample lengths for target apprentice learning. We compare TA-TL with UMA-TL and generic-RL on target task. We examine the efficiency and effectiveness of transfer methods based on sample efficiency in learning the target task and speed of convergence to maximum average reward. Similar to same domain transfer Figure \ref{fig:avg_reward_MC} and \ref{fig:solutiontime_MC} shows the quality of transfer for TA-TL through faster convergence to average maximum reward with lesser training samples compared to UMA-TL and RL methods.

\subsection{Negative transfer}
In our last result, we demonstrate that the proposed transfer is robust to negative transfers. Given a target model, the effectiveness of transfer depends on the relevance of the source task to the target task. If the relationship is strong, the transfer method can take advantage of it, significantly improving the performance of the target task through transfer. However, if the source and target are not sufficiently related or the features of source task do not correspond to the target, the transfer may not improve or even decrease the performance in target task leading to negative transfer.

We show that the UMA-TL suffers from a negative transfer in this results,  where as the performance of presented TA-TL is much superior compared to UMA-TL and RL(no transfer). We demonstrate this through an inverted pendulum upright balance task. We use inverted pendulum model as both source and target systems. The target is different from source model in the sign of the control action. With exactly same dynamics in both source and the target model, but with the sign flipped of the control effective term in the target, we observed that an initialized target task learning (UMA-TL) suffers with negative transfer. RL is indifferent to sign change as it learns policy from scratch. Whereas for the TA-TL method, since we learn the apprentice model to the target, we learn the sign associated with action as well. Thereby the policy modification term sign is flipped accordingly, and same policy transfer performance is achieved irrespective of the control sign change.

Figure \ref{fig:avg_reward_NT} and \ref{fig:solutiontime_NT} shows the quality of transfer through faster convergence to average maximum reward with lesser training samples for proposed TA-TL method compared to UMA-TL and RL methods. It is to be observed that UMA-TL method converges to much lower average reward and gets stuck in a local minima and never achieves the upright balance of pendulum. Also, the samples observed by UMA-TL in learning the task is much higher compared to no transfer (RL) and proposed TA-TL methods. 
\section{Conclusions}
We introduced a new Transfer Learning technique in RL, which leads to sample efficient transfer between source and target tasks. The presented approach demonstrates the near-optimality of the transferred policy in target domain by augmenting it with an adaptive policy; which accounts for the model error between target and projected source. The sample complexity of the transfer is reduced to target apprentice learning, which we demonstrated empirically, leads to more than one order improvement in training lengths over existing approaches.


\bibliographystyle{unsrt}
\bibliography{TL_references}

\end{document}